\documentclass[twoside]{article}

\usepackage[preprint]{aistats2027}
\usepackage{amsmath,amssymb,amsthm,booktabs}
\usepackage{graphicx,tikz}
\usepackage{algorithm,algorithmic}
\usetikzlibrary{arrows.meta}
\definecolor{neuralgreen}{HTML}{008577}
\definecolor{neuralgreenfill}{HTML}{E7F5F1}
\definecolor{neuralpurple}{HTML}{8053A0}
\definecolor{neuralpurplefill}{HTML}{F3EDF8}
\usepackage[round]{natbib}
\defcitealias{timss2023grade4}{IEA, 2024}
\usepackage[hidelinks]{hyperref}
\newtheorem{proposition}{Proposition}

\newcommand{\E}{\mathbb E}

\newcommand{\Normal}{\mathcal N}
\newcommand{\KL}{\operatorname{KL}}
\newcommand{\TV}{\mathrm{TV}}

\begin{document}
\setlength{\textfloatsep}{20pt plus 2pt minus 4pt}
\setlength{\floatsep}{8pt plus 2pt minus 2pt}
\setlength{\intextsep}{10pt plus 2pt minus 2pt}
\runningtitle{Semi-Amortized Bayesian Inference}
\runningauthor{Stevenson, B{\"u}rkner, Heathcote, and Radev}
\twocolumn[
\aistatstitle{Semi-Amortized Bayesian Inference with Reusable Local Posteriors}
\aistatsauthor{Niek Stevenson \And Paul-Christian B{\"u}rkner \And Andrew Heathcote \And Stefan T.~Radev}
\aistatsaddress{University of Amsterdam \And TU Dortmund University \And University of Amsterdam \And Rensselaer Polytechnic\\Institute}
]

\begin{abstract}
Amortized Bayesian inference uses simulations to perform Bayesian estimation when likelihoods are expensive or analytically unavailable. However, learning the full joint posterior is challenging in high-dimensional models with complex dependencies. We introduce a semi-amortized method that divides parameters into smaller blocks, and uses conditional independence to identify the information needed to infer each block. We train neural estimators for blocks whose posterior updates are computationally expensive or frequently repeated. Each estimator learns from simulations of the relevant model components to approximate the block’s full conditional posterior, given the relevant observations and the values of the variables on which it depends. A Markov chain then alternates direct draws from these estimators with conventional updates of the remaining variables. Each neural update uses the current conditioning values and requires neither likelihood evaluation nor an acceptance–rejection step. Where the same conditional structure recurs, one estimator can be reused across blocks and datasets, with an output dimension determined by the block rather than the full model. We illustrate the approach in high-dimensional and computationally expensive hierarchical and spatiotemporal models.
\end{abstract}

\section{INTRODUCTION}
\begin{figure*}[t]
\centering
\includegraphics[width=0.99\textwidth]{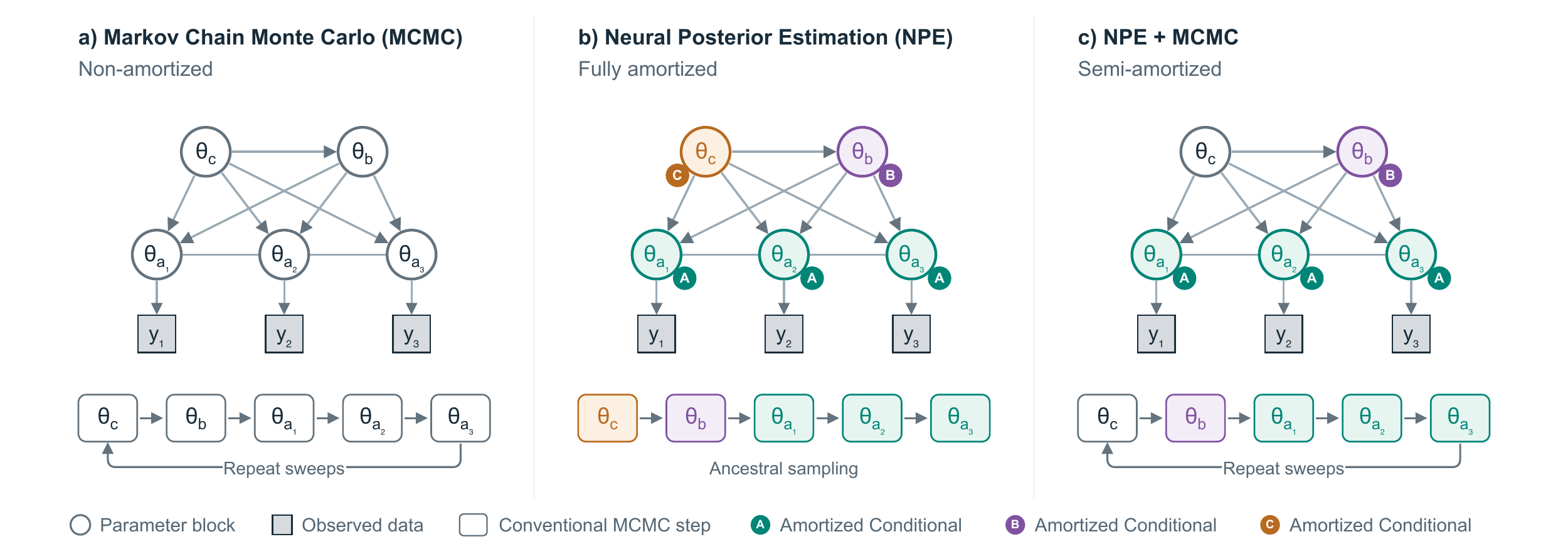}
\caption{Comparison of sampling methods for the same graphical model, where shared parameters $\theta_c$ and $\theta_b$ govern interconnected local parameters $\theta_{a_i}$, each generating observations $y_i$. \textbf{(a)} Conventional MCMC updates all parameter blocks in repeated sweeps. \textbf{(b)} Amortized neural posterior estimation (NPE) tries to approximate the full joint posterior directly through a suitable inverse factorization (see text) without resorting to MCMC. \textbf{(c)} Semi-amortized inference (ours) combines conventional MCMC updates of $\theta_c$ with draws from amortized conditional posteriors for the remaining $\theta_{-c}$. Estimator A is shared across the $\theta_{a_i}$ blocks, while estimator B updates $\theta_b$. During each sweep, a conditional update uses the current values of the other blocks.}
\label{fig:inference-overview}
\end{figure*}

Neural simulation-based inference (SBI) can learn intractable probability distributions from simulations when likelihoods are expensive or analytically unavailable~ \citep{cranmer2020frontier}. In amortized neural posterior estimation (NPE), a conditional density estimator is trained on simulated parameter--data pairs $(\theta, y)$ to approximate the posterior $p(\theta \mid y)$. Once trained, it can obtain samples for any new observation $y^*$ without retraining, amortizing the cost of repeated inference.

However, amortizing the full posterior becomes challenging as the number of inferred parameters grows and their dependencies become more complex. For example, \textit{hierarchical models} connect group-specific parameters through population distributions \citep{gelman2013bda}, \textit{state-space models} link latent states over time \citep{durbin2012statespace}, and \textit{spatial models} describe dependence across locations \citep{rue2005gmrf}. Such models can strain amortized methods, as the posterior dimension grows with the number of groups, time points, or locations, while dependencies between parameters must still be represented jointly.

The joint posterior of such models can often be decomposed into \textit{inverse factors} involving subsets of variables \citep{Stuhlmueller2013, Webb2018, habermann2026graphical}. This decomposition exposes smaller inference problems in which one block of variables is inferred conditional on the others. In a hierarchical model, for example, each group's parameters can be inferred from its observations conditional on the population parameters. Classical Gibbs sampling exploits this structure by iteratively drawing each parameter block from its full conditional distribution.

Amortized Bayesian inference on structured models also relies on clever factorization. Existing methods assign neural estimators to individual inverse factors so that samples from the joint posterior can be generated by ancestral sampling \citep{Webb2018, heinrich2024hierarchical, habermann2025multilevel, habermann2026graphical, arruda2024amortized}. While effective for moderately sized models, full amortization can become extremely costly when the posterior contains thousands or even hundreds of thousands of factors due to simulation and memory costs. Compositional modeling mitigates this problem by combining exchangeable factors at inference \citep{arruda2026compositional}. However, its accuracy can deteriorate as the number of composed factors grows.

Building on likelihood-free approximate Gibbs sampling \citep{rodrigues2020likelihood}, we instead develop a \emph{semi-amortized} Bayesian inference (SABI) method that combines neural conditional updates with conventional MCMC steps (\autoref{fig:inference-overview}).
Similar to \citet{arruda2026compositional}, our method does not simulate the full graphical model; it requires only identifying which inverse factors are difficult to sample and which admit efficient conventional updates—a distinction that is typically clear from the model structure. In hierarchical models, for example, the challenging step is often sampling group-specific posteriors, whereas the hyperposterior is well identified and easier to sample.

We demonstrate  SABI across three realistic applications with complex dependencies and tens of thousands of parameters, where comparisons with reference samplers show close agreement in posterior summaries. In a simulation-based application, SABI outperforms alternative SBI methods; further, it enables reuse of the same trained estimator across different hierarchical structures at scale. Additionally, we provide an initial analysis of the sampler's stationary accuracy.


\section{SEMI-AMORTIZED INFERENCE}
\label{sec:method}

Let $\theta=(\theta_1,\ldots,\theta_N)$ collect the parameter blocks selected for learned updates, $\vartheta = (\vartheta_1,\ldots,\vartheta_M)$ the remaining parameter blocks, $y$ the observations, and $x$ the fixed model inputs. The full conditional of each block $\theta_i$ depends on the other model variables only through its Markov blanket \citep{koller2009pgm}. We denote the observations in this blanket by $y_i$, collect the blanket parameter values in $c_i$, and write $x_i$ for the relevant fixed inputs. For example, in a hierarchical two-level model, $c_i$ contains the population parameters, while $x_i$ contains the group's design matrix. In a spatial model, $c_i$ may contain neighboring states and shared parameters. We write the full conditional as:
\begin{equation}
p(\theta_i\mid y,\theta_{-i},\vartheta,x)
=p(\theta_i\mid y_i,c_i,x_i),
\label{eq:block-conditional}
\end{equation}
where $\theta_{-i}$ denotes the remaining learned blocks.

\subsection{Training the Neural Updates}
We train a neural estimator to approximate this conditional using existing SBI infrastructure:
\begin{equation}
q_{\phi_i}(\theta_i\mid y_i,c_i,x_i)
\approx p(\theta_i\mid y_i,c_i,x_i),
\label{eq:block-context}
\end{equation}
where $\phi_i$ denotes the learnable neural network parameters.
The estimator must learn how a block's posterior changes with its conditioning parameters and observation design. We therefore generate training pairs of conditioning parameters and fixed inputs from a chosen training distribution $p(c_i,x_i)$. For each combination, we draw the block's parameters from the model's conditional prior and simulate its observations:
\begin{equation}
\begin{aligned}
(c_i^{(s)},x_i^{(s)})&\sim p(c_i,x_i),\\
\theta_i^{(s)}&\sim p(\theta \mid c_i^{(s)},x_i^{(s)}),\\
y_i^{(s)}&\sim p(y \mid \theta_i^{(s)},c_i^{(s)},x_i^{(s)}).
\end{aligned}
\label{eq:local-simulation}
\end{equation}
For simulated training set $s$. The conditional prior and observation model determine the posterior target, while $p(c_i,x_i)$ controls training coverage. Varying population means, scales, covariates, and observation counts supports reuse across groups and datasets. Preliminary inference or conditional diagnostics can guide refinement of $p(c_i,x_i)$ to improve approximation accuracy of $q_i(\theta_i\mid y_i,c_i,x_i)$ see Appendix~\ref{app:refine}.

Large or variable-sized inputs of $y_i$, $c_i$, and $x_i$ can be represented by a learned summary network $h_{\psi_i}(y_i,c_i,x_i)$ with network parameters $\psi_i$. We train the estimator $q_{\phi_i}$ and any learned summary $h_{\psi_i}$ by minimizing the expected negative log-likelihood over simulated training examples \citep{radevBayesFlow2020}:
\begin{equation}
\label{eq:estimator}\\
\mathcal L(\phi_i,\psi_i) = \E[-\log q_{\phi_i}(\theta_i\mid h_{\psi_i}(y_i,c_i,x_i))].
\end{equation}

Training can use transformed parameter coordinates, with the corresponding density transformation and equivalent objective
described in Appendix~\ref{app:coordinates}. We use normalizing flows as implemented in \texttt{BayesFlow}  \citep{kuhmichel2026bayesflow}. We chose normalizing flows due to their fast single-step sampling, but any modern free-form architecture is compatible with our method \citep[e.g., diffusion or flow matching][]{arruda2025diffusion}.


\subsection{Semi-Amortized Inference}
After training, inference proceeds through an MCMC sampler on all model parameters $(\theta, \vartheta)$. Starting from initial values, the sampler repeatedly updates each block conditional on the current values of the others. A complete pass through the learned blocks $\theta_1,\ldots,\theta_N$ and the conventional blocks $\vartheta_1,\ldots,\vartheta_M$ is one \emph{sweep}. We write $q_i$ for the trained estimator assigned to block $i$; multiple blocks may share the same estimator. The network parameters $(\phi_i, \psi_i)$, observations $y$, and fixed inputs $x$ remain unchanged throughout sampling.

For a learned update, we first compute $c_i$ from the latest values of the other parameters, including any values updated earlier in the sweep. The observations $y_i$, conditions $c_i$, and fixed inputs $x_i$ enter the trained estimator through the input map in~\eqref{eq:estimator}. We then sample a draw from $q_i(\theta_i \mid y_i,c_i,x_i)$ to update block $\theta_i$.

Each conventional block $\vartheta_j$ is updated using a Gibbs draw from its full conditional $p(\vartheta_j\mid\theta,\vartheta_{-j},y,x)$ or an MCMC transition preserving that distribution, such as Metropolis--Hastings or slice sampling. Here $\vartheta_{-j}$ denotes the other conventional blocks. 

\textbf{Algorithm}~\ref{alg:sampler} outlines the resulting sampler. It discards $W$ warm-up sweeps and retains the following $T$ states. Each update immediately replaces the corresponding current value. 

\begin{algorithm}[t]
\caption{Semi-amortized block sampling}
\label{alg:sampler}
\begin{algorithmic}[1]
\REQUIRE Observations $y$, fixed model inputs $x$, trained estimators $q_1,\ldots,q_N$, initial parameters $(\theta^{(0)},\vartheta^{(0)})$, $W$ warm-up sweeps, and $T$ retained sweeps.
\FOR{$t=1,\ldots,W+T$}
  \STATE $(\theta^{(t)},\vartheta^{(t)})\gets(\theta^{(t-1)},\vartheta^{(t-1)})$.
  \FOR{$i=1,\ldots,N$}
\STATE Form $c_i^{(t)}$ from $\theta_{-i}^{(t)}$ and $\vartheta^{(t)}$.
    \STATE Draw $\theta_i^{(t)}\sim q_i(\theta_i \mid y_i,c_i^{(t)},x_i)$.
  \ENDFOR
  \FOR{$j=1,\ldots,M$}
    \STATE Draw $\vartheta_j^{(t)} \sim p(\vartheta_j\mid\theta^{(t)},\vartheta_{-j}^{(t)},y,x)$ via Gibbs or MCMC transition.
  \ENDFOR
  \IF{$t>W$}
    \STATE Store $S_{t-W}=(\theta^{(t)},\vartheta^{(t)})$.
  \ENDIF
\ENDFOR
\STATE \textbf{Return} $S_1,\ldots,S_T$.
\end{algorithmic}
\end{algorithm}


\section{ACCURACY OF THE PROPOSED SAMPLER}
\label{sec:theory}
Approximation in the learned conditionals can affect posterior accuracy. We first establish posterior invariance when the learned conditionals are exact, then bound posterior error over a finite number of sweeps in terms of conditional approximation error. We also show that small posterior-averaged conditional error does not guarantee stationary accuracy. Proofs and supplementary derivations are in Appendix~\ref{app:proofs}.

\subsection{Posterior Invariance}
For fixed observations $y$ and model inputs $x$, let $\Pi$ denote the joint posterior of $u=(\theta,\vartheta)$. Let $\widetilde P$ be the transition kernel for one sweep of \textbf{Algorithm}~\ref{alg:sampler}, and $P$ the same kernel with each learned update replaced by an exact full-conditional draw. The exact sampler preserves the posterior, $\Pi P=\Pi$. Thus, when every learned conditional is exact, SABI preserves the intended joint posterior, including dependencies between blocks (Appendix~\ref{app:joint-invariance}). Next we examine how approximation error can affect posterior accuracy.

\subsection{Finite-Sweep Error}
For learned block $i$, define its conditional KL error at parameter state $u$ and the total error averaged over the posterior:
\begin{equation}
\begin{aligned}
\kappa_i(u)
&=\KL\{p(\cdot\mid y_i,c_i(u),x_i)
       \,\|\,q_i(\cdot\mid y_i,c_i(u),x_i)\},\\
D_y&=\sum_{i=1}^N\E_\Pi[\kappa_i(u)].
\end{aligned}
\label{eq:conditional-error}
\end{equation}
Here $c_i(u)$ denotes the conditioning inputs computed from $u$. $D_y$ sums the errors across learned blocks, averaging over the exact posterior for the observed data $y$. To isolate the error introduced by learned updates, we compare them with exact conditional draws. Starting from $\Pi$, the exact sampler preserves $\Pi$, whereas the approximate sampler has distribution $\Pi\widetilde P^t$ after $t$ sweeps. This lets us bound the finite-sweep error.

\begin{proposition}[Finite-sweep error]
\label{prop:finite-sweep}
Suppose $D_y<\infty$. For every integer $t\ge1$,
\begin{align}
\KL(\Pi\,\|\,\Pi\widetilde P^t)
&\le tD_y,
\label{eq:finite-sweep-kl}\\
\|\Pi-\Pi\widetilde P^t\|_{\TV}
&\le\min\{1,\sqrt{tD_y/2}\}.
\label{eq:finite-sweep-tv}
\end{align}
\end{proposition}
Since every exact block update preserves $\Pi$, the expected conditional KL errors accumulate to $tD_y$ over $t$ sweeps, giving~\eqref{eq:finite-sweep-kl} (Appendix~\ref{app:finite-sweep}). Equation~\eqref{eq:finite-sweep-tv} shows the a bound on posterior probability errors, where the approximate probability differs from the intended probability by at most $\min\{1,\sqrt{tD_y/2}\}$. Here $\|\cdot\|_{\TV}$ denotes total variation distance.

Training minimizes the KL averaged over simulated datasets and conditioning inputs, whereas the bound depends on observed dataset through $D_y$. Its posterior may concentrate on conditioning values where the learned updates are less accurate, despite small average error across training simulations. Small training error therefore need not imply small $D_y$, and the following proposition shows that small $D_y$ itself need not imply stationary accuracy.

\begin{proposition}[Conditional error and stationary accuracy]
\label{prop:stationary-error}
For every $\delta>0$, there is a posterior $\Pi$ over two parameter blocks $(\theta,\vartheta)$ and an approximate sampler alternating learned updates of $\theta$ with an exact updates of $\vartheta$, whose stationary distribution $\widetilde\Pi$ satisfies
\begin{equation}
D_y<\delta,
\qquad
\|\widetilde\Pi-\Pi\|_{\TV}>1-\delta.
\end{equation}
Both the approximate sampler and its exact counterpart converge to unique stationary distributions.
\end{proposition}

Thus, arbitrarily small conditional error averaged under $\Pi$ does not guarantee an accurate stationary distribution. The learned update is inaccurate at conditioning values that are rare under $\Pi$ and contribute little to $D_y$ (\ref{app:stationary-error}). The approximate updates make these values difficult to leave, so the chain spends most of its time there. Errors that receive little weight in $D_y$ can thus dominate the approximate sampler's behavior. Errors can also propagate through the conditioning inputs supplied to subsequent updates, as illustrated for Gaussian conditionals in Appendix~\ref{app:gaussian-conditional}. 

More generally, the effect of conditional errors on individual posterior expectations depends on their persistence and cancellation across updates (Appendix~\ref{app:poisson}). Our results on real data at scale show that such errors are less concerning in practice and can be reduced with longer training runs.

\section{RELATED WORK}
\label{sec:rel_work}

The literature on graphical model estimation is vast. The paragraphs below focus exclusively on amortized approaches explicitly situated in or adjacent to SBI.

\paragraph{Structured posterior approximations.}
Graph-structured inference methods use the generative model to construct posterior approximations. Faithful inversion derives inference-network factorizations that avoid spurious conditional independencies \citep{Webb2018, habermann2026graphical}. In contrast, \citet{ambrogioni2021automatic} and \citet{rouillard2022adavi} adopt simpler variational factorizations that need not preserve all posterior dependencies, retaining the prior dependency structure and imposing mean-field independence across latent blocks, respectively. Importantly, these approaches target full amortization rather than selected conditional transitions.

\paragraph{Hierarchical neural posterior estimation.}
Relatedly, hierarchical neural posterior estimators exploit posterior factorizations, sharing networks across groups and respecting exchangeability \citep{rodrigues2021hnpe, heinrich2024hierarchical, rouillard2023pavi, habermann2025multilevel, habermann2026graphical}. Fully amortized variants learn population- and group-level posterior factors, but require complete hierarchical datasets for training. \citet{arruda2024amortized} instead amortize group-level posteriors and fit population-level parameters with a maximum likelihood estimator. More recently, \citet{arruda2026compositional} avoid complete hierarchical simulations by composing scores of exchangeable factors at the cost of error accumulation, while still pursuing fully amortized inference.

\paragraph{Amortized Monte Carlo proposals.}
Another line of work learns neural proposals while retaining Monte Carlo correction.
Inference compilation and its extensions use neural proposals in sequential importance sampling \citep{le2017inference,baydin2019efficient,harvey2022attention}.
Meta-learned block proposals and lightweight inference compilation instead retain Metropolis--Hastings correction \citep{wang2018meta,liang2021accelerating}, while amortized population Gibbs sampling uses importance weighting and resampling \citep{wu2020apg}.
These corrections require computable ratios involving model and proposal densities, including the relevant prior and likelihood contributions, potentially evaluated on an augmented simulator trace.
We instead use learned full-conditional approximations directly as sampling updates, without acceptance tests or importance weights.

\paragraph{Likelihood-free approximate Gibbs.}
The closest work is the likelihood-free approximate Gibbs sampler of \citet{rodrigues2020likelihood}, which replaces intractable full conditionals with neural regression. We extend this backbone with modern SBI machinery (i.e., flexible neural estimators and learned summaries), and analyze sampler correctness and the propagation of approximation errors. Further, we provide benchmarks against gold-standard samplers and demonstrate scalability to large hierarchical and spatiotemporal models, including a simulation-based model with over 30,000 parameters and a spatiotemporal model with nearly two million observations.

\section{EXPERIMENTS}

\subsection{TIMSS: Three-level hierarchical model}
\label{sec:timss}

Our first application evaluates the accuracy and scalability of SABI in a large 21,757 parameter three-level hierarchical model of the TIMSS dataset, with 256,416 students nested in 10,820 schools across 54 countries \citepalias{timss2023grade4}. We model fourth-grade mathematics achievement as a function of home resources. The response is the first of  five plausible values, generated through item response theory \citep{bezirhan2024timss}. We train neural estimators for conditional updates for both school and country coefficients. The model's analytic full conditionals allow direct comparison with a Gibbs sampler, providing an \textit{oracle reference} for posterior accuracy and computational efficiency. Data preparation, priors, and sampling details are given in the supplementary material.

We use a three-level Gaussian linear model with students nested within schools and countries. School intercepts and home-resources slopes vary around country coefficients, which in turn vary around population means (Appendix \ref{app:timss}). The model estimates intercept-slope correlations at both the school and country levels. Figure~\ref{fig:timss-graph} shows a graphical representation of the model.

\begin{figure}[!htbp]
\centering
\begin{tikzpicture}[
  x=1.3cm,y=0.82cm,>=Stealth,
  latent/.style={circle,draw,minimum size=7.5mm,inner sep=1pt,font=\small},
  neuralA/.style={latent,draw=neuralgreen,fill=neuralgreenfill},
  neuralB/.style={latent,draw=neuralpurple,fill=neuralpurplefill},
  observed/.style={latent,fill=black!15},
  plate/.style={draw,rounded corners=2pt},
  every path/.style={line width=0.55pt}
]
  \draw[plate] (-0.35,-0.65) rectangle (4.75,-5.20);
  \draw[plate] (-0.10,-1.85) rectangle (4.45,-4.70);
  \draw[plate] (0.15,-2.95) rectangle (4.15,-4.20);
  \node[anchor=east,font=\scriptsize] at (4.70,-4.95) {$c=1,\ldots,C$};
  \node[anchor=east,font=\scriptsize] at (4.40,-4.45) {$s=1,\ldots,S_c$};
  \node[anchor=south east,font=\scriptsize] at (4.10,-4.12) {$j=1,\ldots,n_{cs}$};

  \node[latent] (mu) at (0.25,0) {$\mu$};
  \node[latent] (countrycov) at (1.55,0) {$\Sigma_C$};
  \node[latent] (schoolcov) at (3.00,0) {$\Sigma_S$};
  \node[latent] (noise) at (4.40,0) {$\sigma^2$};
  \node[neuralB] (country) at (0.90,-1.25) {$b_c$};
  \node[neuralA] (school) at (2.15,-2.40) {$\theta_{cs}$};
  \node[observed] (x) at (0.75,-3.65) {$x_{csj}$};
  \node[observed] (y) at (2.15,-3.65) {$y_{csj}$};

  \draw[->] (mu) -- (country);
  \draw[->] (countrycov) -- (country);
  \draw[->] (country) -- (school);
  \draw[->] (schoolcov) -- (school);
  \draw[->] (school) -- (y);
  \draw[->] (noise) -- (4.40,-1.60) -- (3.80,-2.70) -- (y);
  \draw[->] (x) -- (y);
\end{tikzpicture}
\caption{TIMSS hierarchy fit in \textbf{Experiment 1}. Students nested within schools and countries/systems. Separate neural estimators update school coefficients $\theta_{cs}=(\alpha_{cs},\beta_{cs})$ (green) and country coefficients $b_c=(\alpha_c,\beta_c)$ (purple). Population means $\mu=(\mu_\alpha,\mu_\beta)$, covariance matrices $\Sigma_C$ and $\Sigma_S$, and residual variance $\sigma^2$ receive Gibbs updates. Gray nodes denote the achievement scores ($y$), and the home-resources scores ($x$).}
\label{fig:timss-graph}
\end{figure}
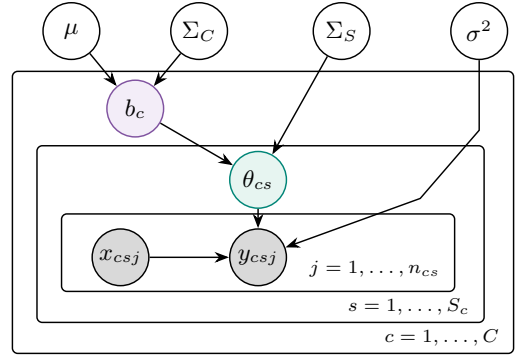

Inference combines neural draws for school and country coefficients with conjugate Gibbs updates for the remaining parameters. The school network conditions on each school's data, the country coefficients, and the residual variance. The country network conditions on the school coefficients, the global parameters, and the residual variance.

\begin{figure}[!htb]
\centering
\includegraphics[trim=1bp 23bp 0bp 10bp,clip,width=\linewidth]{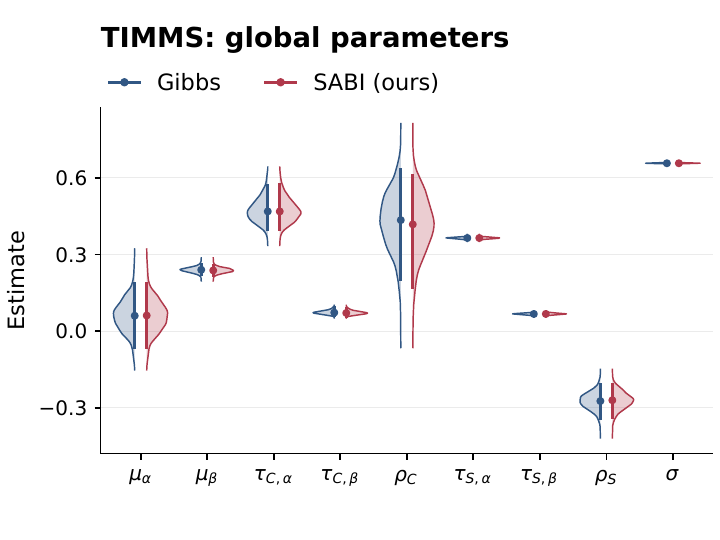}
\par\smallskip
\includegraphics[trim=4bp 20bp 5bp 3bp,clip,width=0.49\linewidth]{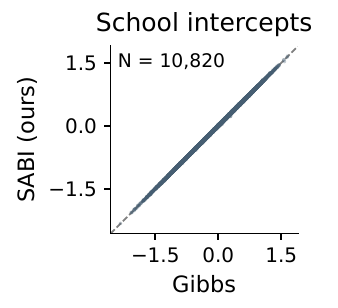}\hfill
\includegraphics[trim=4bp 20bp 5bp 3bp,clip,width=0.49\linewidth]{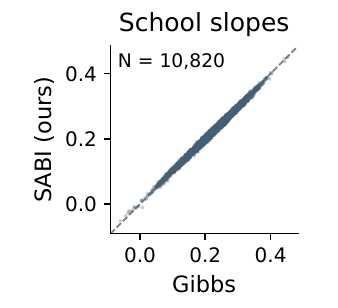}
\par\smallskip
\includegraphics[trim=4bp 3bp 5bp 3bp,clip,width=0.49\linewidth]{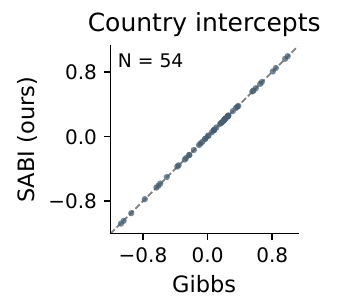}\hfill
\includegraphics[trim=4bp 3bp 5bp 3bp,clip,width=0.49\linewidth]{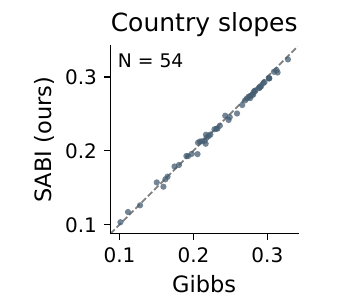}
\caption{TIMSS posterior comparison in \textbf{Experiment 1}. Top: global posterior medians (dots) and 95\% credible intervals (bars), using Gibbs (blue) or SABI (red); $C$ and $S$ denote country and school levels. Bottom: paired posterior medians for every school and country intercept and slope, with Gibbs estimates on the x axis and SABI estimates on the y axis. We observe high agreement across all parameters.}
\label{fig:timss-global}
\end{figure} 

We compare SABI to a full-Gibbs sampler that replaces the two neural updates with draws from their analytic Gaussian conditionals. Full-Gibbs was 10.4 times faster than SABI. \autoref{fig:timss-global} compares the nine global marginals and the school and country parameter medians for the full dataset, showing close marginal agreement between the two samplers. Global agreement was further corroborated by a classifier two-sample test \citep[C2ST;][]{lopez-pazRevisitingClassifier2018}, showing an accuracy score of $0.508$ across the joint global parameter distribution, indicating indiscriminate draws, although there was still mismatch in the full joint including all locals (accuracy: $0.764$).

\subsection{Citi Bike: Spatial and Temporal Dependence}
\label{sec:citibike}

Our second application evaluates SABI in a large spatiotemporal model with an expensive observation likelihood. We model hourly Citi Bike departures at 1,427 stations over eight weeks, comprising 11,416 station--weeks and 1,917,888 observations \citep{citibike2024data}. 

We use a Poisson–lognormal distribution, with an expected weekly departures and overdispersion parameter for every station in every week. A Gaussian graph prior links nearby stations and consecutive weeks (Figure~\ref{fig:citibike-graph}). An hourly profile is shared across stations and weeks that captures the distribution of hourly departures across the week (see Appendix~\ref{app:citibike}). The observation model introduces a latent rate for each count; critically, \textit{SABI avoids the need to numerically integrate out or explicitly sample these rates during inference.} 

\begin{figure}[!htbp]
\centering
\begin{tikzpicture}[
  x=0.99cm,y=0.95cm,>=Stealth,
  latent/.style={circle,draw,minimum size=7.5mm,inner sep=1pt,font=\small},
  neuralA/.style={latent,draw=neuralgreen,fill=neuralgreenfill},
  neuralB/.style={latent,draw=neuralpurple,fill=neuralpurplefill},
  state/.style={neuralA},
  observed/.style={latent,fill=black!15},
  plate/.style={draw,rounded corners=2pt},
  every path/.style={line width=0.55pt}
]
  \node[latent] (mu) at (0.25,0) {$\mu$};
  \node[latent] (sigma) at (1.65,0) {$\Sigma$};
  \node[latent] (space) at (3.05,0) {$\rho_S$};
  \node[latent] (time) at (4.45,0) {$\rho_T$};

  \draw[plate] (-0.15,-0.75) rectangle (4.85,-4.50);
  \foreach \name/\xx in {mu/0.25,sigma/1.65,space/3.05,time/4.45}
    \draw[->] (\name) -- (\xx,-0.75);
  \node[font=\scriptsize,anchor=north] at (2.35,-0.86)
    {Station--week Gaussian field};

  \foreach \row/\station/\yy in {1/1/-1.65,2/2/-2.80,3/i/-3.95} {
    \foreach \col/\week/\xx in {1/1/0.55,2/2/2.20,3/t/3.85} {
      \node[state] (v\row\col) at (\xx,\yy) {$\theta_{\station\week}$};
    }
    \draw (v\row1) -- (v\row2);
    \draw (v\row2) -- node[fill=white,inner sep=1pt,font=\scriptsize] {$\cdots$} (v\row3);
  }
  \foreach \col in {1,2,3} {
    \draw (v1\col) -- (v2\col);
    \path (v2\col) -- coordinate[midway] (dots\col) (v3\col);
    \foreach \offset in {-0.65,0,0.65}
      \fill ([yshift=\offset mm]dots\col) circle[radius=0.3pt];
  }
  \node[font=\scriptsize,fill=white,inner sep=1pt] at (1.37,-2.80) {time};
  \node[font=\scriptsize,rotate=90] at (4.48,-2.80) {space};

  \draw[plate] (5.30,-2.86) rectangle (8.06,-4.50);
  \node[anchor=south east,font=\scriptsize] at (7.98,-4.38)
    {$h=1,\ldots,168$};
  \node[neuralB] (u) at (5.88,-1.93) {$u$};
  \node[font=\scriptsize,align=center] at (5.88,-1.10)
    {Shared hourly\\profile};
  \node[latent] (ell) at (5.88,-3.43) {$\ell_{ith}$};
  \node[observed] (y) at (7.47,-3.43) {$y_{ith}$};
  \draw[->] (v33) -- (ell);
  \draw[->] (u) -- (ell);
  \draw[->] (ell) -- (y);
\end{tikzpicture}
\caption{Citi Bike Gaussian field fit in \textbf{Experiment 2}. Departures $y_{ith}$ at station $i$ in week $t$ and hour $h$. Neural estimators update station--week parameters $\theta_{it}=(a_{it},b_{it})^\top$, controlling activity and overdispersion (green), and the shared hourly profile $u$ (purple), that together determine latent activity rates. Population mean $\mu$, covariance $\Sigma$, and spatial and temporal coupling strengths $\rho_S$ and $\rho_T$ receive conventional updates.}
\label{fig:citibike-graph}
\end{figure}
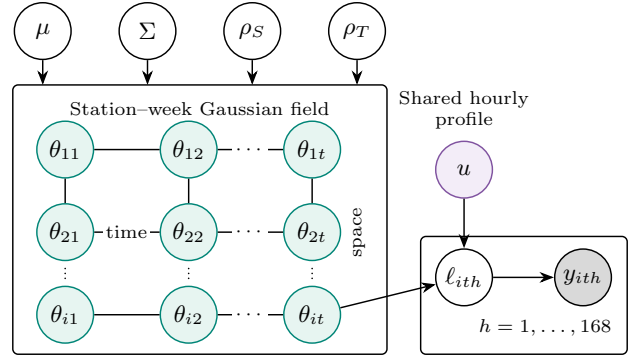

\begin{figure}[!htb]
\centering
\includegraphics[trim=15bp 26bp 0bp 23bp,clip,width=\linewidth]{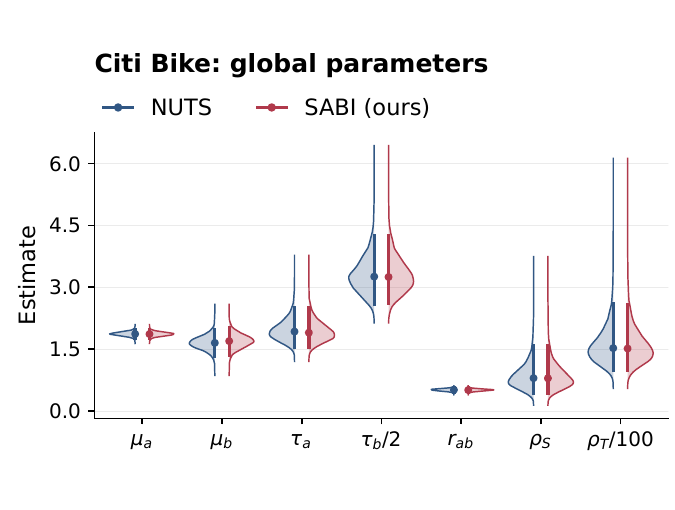}
\par\smallskip
\includegraphics[trim=10bp 13bp 2bp 23bp,clip,width=0.49\linewidth]{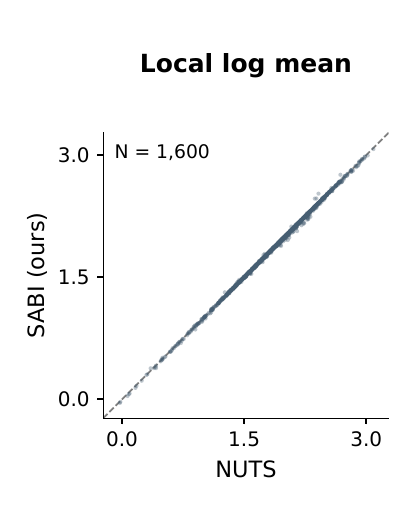}\hfill
\includegraphics[trim=10bp 13bp 2bp 23bp,clip,width=0.49\linewidth]{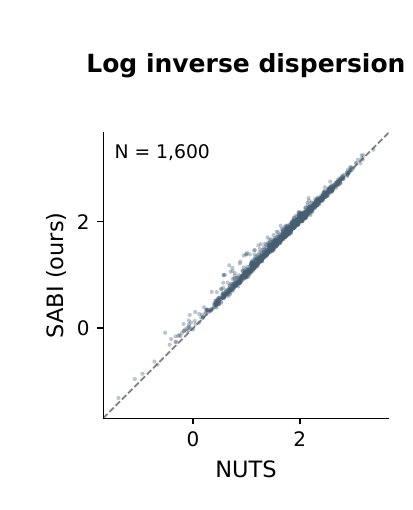}
\caption{Citi Bike posterior comparison for 200 stations in \textbf{Experiment 2}. Top: global posterior medians (dots) and 95\% credible intervals using Stan (blue) or SABI (red). $\tau_b$ is divided by two and $\rho_T$ by 100. Bottom: paired posterior medians of $a_{it}$ and $b_{it}$ for all 1,600 station--weeks, with Stan on the x-axis and SABI on the y-axis. We find high agreement across all parameters, with slight overestimation of local inverse dispersion that did not translate to global bias.}
\label{fig:citibike-global}
\end{figure}

Inference alternates neural draws from the station--week network, the shared hourly profile network and conventional updates of the global parameters. We assess posterior accuracy and computational efficiency against NUTS as implemented in Stan \citep{stan}, using subsets of 12, 36, 100, and 200 stations. \autoref{fig:citibike-efficiency} shows that SABI was more than 3 times more efficient at higher station counts, as measured by effective sample size (ESS) per second.

\autoref{fig:citibike-global} compares global posterior densities and station--week parameter medians for the 200 stations analysis. Global posterior means differ from Stan by at most 0.26 Stan posterior standard deviations. Global agreement was again assessed using C2ST, showing only mild discrimation of the joint distribution of global parameters with an accuracy score of $0.62$. However, the full joint including all locals was perfectly discriminable (accuracy: $1$), likely because of the small but consistent bias in the inverse dispersion parameters.

The full model has 23,006 free parameters after marginalizing the latent log rates. Explicitly sampling these rates increases the parameter count to 1,940,894, exceeding the available memory in our Stan implementation. Numerical marginalization was prohibitively slow in Stan, and we stopped that run after one day. SABI avoids both costs through learned updates, requiring neither explicit sampling nor integration. Four chains, each with 500 burn-in sweeps and 1,250 retained draws, took 157 minutes. The maximum $\widehat R$ was 1.088; median bulk and tail ESS \citep{vehtari2021rhat} were 4,005 and 4,491, with minima of 38 and 144. Further posterior predictive checks are included in Appendix~\ref{app:citi-pred}.

\subsection{DMC: Simulation-Based Inference}
\label{sec:dmc}
Our third application evaluates SABI for an observation model \textit{lacking a closed form likelihood} and demonstrates reuse of the same network across different hierarchical structures. We study decision-making in conflict tasks, in which irrelevant information can support or oppose the required response, using the Diffusion Model for Conflict Tasks \citep[DMC;][]{ulrich2015dmc}.

\begin{figure}[!htbp]
\centering
\includegraphics[width=\linewidth]{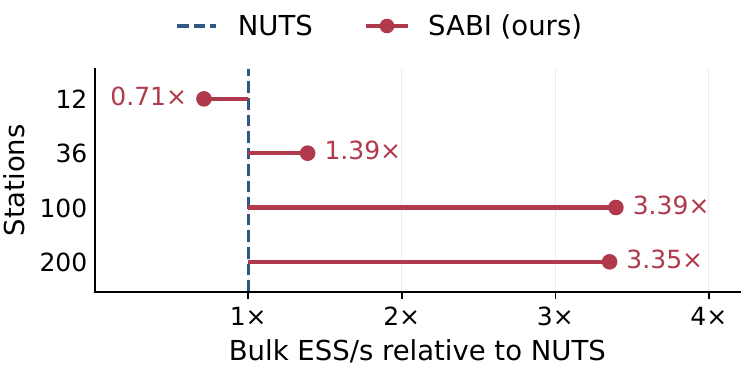}
\caption{Citi Bike sampling efficiency in \textbf{Experiment 2}: median bulk ESS/s across shared parameters for SABI relative to NUTS. The dashed line marks equal efficiency; timing includes warmup.}
\label{fig:citibike-efficiency}
\end{figure}

\begin{figure}[!htbp]
\centering
\begin{tikzpicture}[
  x=1cm,y=0.83cm,>=Stealth,
  latent/.style={circle,draw,minimum size=7.5mm,inner sep=1pt,font=\small},
  neuralA/.style={latent,draw=neuralgreen,fill=neuralgreenfill},
  observed/.style={latent,fill=black!15},
  plate/.style={draw,rounded corners=2pt},
  every node/.style={inner sep=1pt},
  every path/.style={line width=0.55pt}
]
  \draw[plate] (0.45,-0.55) rectangle (5.20,-5.20);
  \node[anchor=south east,font=\scriptsize] at (5.15,-5.12) {$p=1,\ldots,P$};
  \draw[plate] (1.80,-0.85) rectangle (6.70,-5.85);
  \node[anchor=south east,font=\scriptsize] at (6.65,-5.77) {$t=1,\ldots,T$};
  \draw[plate] (2.10,-1.65) rectangle (4.95,-4.70);
  \node[anchor=south east,font=\scriptsize] at (4.90,-4.62) {$d=1,\ldots,D_{pt}$};
  \draw[plate] (2.30,-2.80) rectangle (4.75,-4.20);
  \node[anchor=south east,font=\scriptsize] at (4.70,-4.12) {$N_{dpt}$ participants};
  \node[latent] (pubscale) at (1.10,0) {$\omega_P^2$};
  \node[latent] (sigma) at (2.75,0) {$\sigma^2$};
  \node[latent] (datascale) at (4.00,0) {$\omega_D^2$};
  \node[latent] (taskscale) at (5.90,0) {$\omega_T^2$};
  \node[latent] (publication) at (1.10,-1.10) {$u_p$};
  \node[latent] (dataset) at (4.00,-2.15) {$v_{dpt}$};
  \node[latent] (task) at (5.90,-1.35) {$w_t$};
  \node[latent] (mu) at (-0.05,-3.35) {$\mu$};
  \node[neuralA] (participant) at (2.95,-3.35) {$\theta_i$};
  \node[observed] (data) at (4.15,-3.35) {$y_i$};
  \draw[->] (mu) -- (participant);
  \draw[->] (publication) -- (participant);
  \draw[->] (task) -- (5.40,-2.65) -- (3.60,-2.75) -- (participant);
  \draw[->] (pubscale) -- (publication);
  \draw[->] (datascale) -- (dataset);
  \draw[->] (taskscale) -- (task);
  \draw[->] (dataset) -- (participant);
  \draw[->] (sigma) -- (participant);
  \draw[->] (participant) -- (data);
\end{tikzpicture}
\caption{ACDC hierarchy of \textbf{Experiment 3}. Overlapping plates represent publications and tasks, with $D_{pt}$ datasets in each combination. Parameter nodes collect six coordinates; gray $y_i$ contains a participant's response times and choices. Fixed trial congruencies are omitted. Light green indicates updates from the participant network; hierarchy parameters receive conventional updates.}
\label{fig:acdc-graph}
\end{figure}
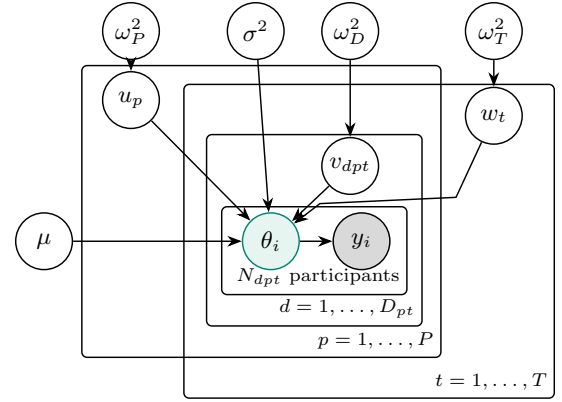

\begin{figure}[!htb]
\centering
\includegraphics[width=\linewidth]{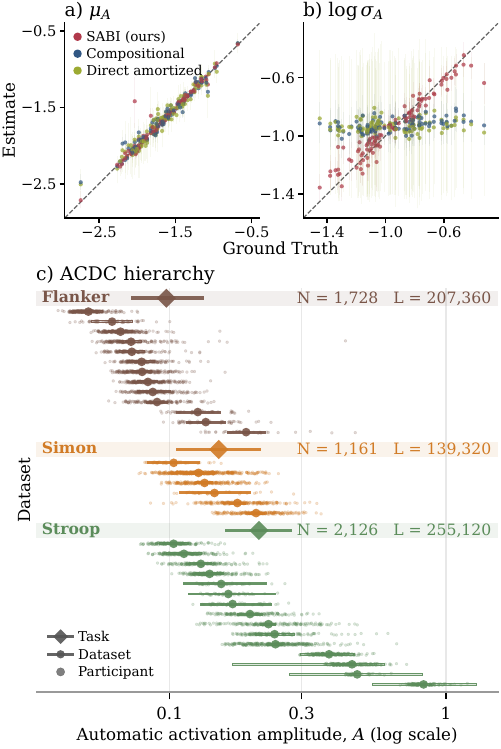}
\caption{DMC estimation and ACDC inference of \textbf{Experiment 3}. (a,b) Posterior means and 95\% credible intervals for the mean $\mu_A$ and log standard deviation $\log\sigma_A$ of $\log A$ across 100 paired datasets. The three methods show similar recovery for the means, but SABI vastly outperforms the other two on the variance. (c) Posterior medians and 95\% intervals for task- and dataset-level population medians of $A$; dots are participant posterior medians. $N$ denotes the number of participants and $L$ the number of trials.}
\label{fig:acdc-results}
\end{figure}

The likelihood function of the DMC is unknown, therefore we cannot compare directly with conventional MCMC methods. Instead, we compare SABI with two alternative SBI approaches, compositional diffusion \citep{arruda2026compositional} and direct amortized posterior estimation \citep{habermann2025multilevel}, using a hierarchical model in which $\theta_{ik}\sim\Normal(\mu_k,\sigma_k^2)$ for each of the six participant parameters. Training budget was closely matched across all three methods, using 26,200,000 ($\pm 0,05\%$) participant simulations, see appendix Appendix~\ref{app:dmc} for details

We evaluated all three methods on 100 paired datasets of 500 participants. Table~\ref{tab:dmc-times} reports RMSE and expected calibration error \citep[ECE;][]{ardizzoneAnalyzingInverse2019} averaged over the six population means and variances, alongside training time and sample efficiency. SABI achieves lower RMSE than both baselines. Direct amortization yields low ECE, but mostly due to poor posterior updating, particularly for variances, leading to just returning the prior (see also panels a and b of Figure~\ref{fig:acdc-results}). Compositional inference has both higher RMSE and higher ECE. Three of 100 datasets had $\widehat{R}>1.1$ (maximum 1.52) using SABI; excluding these datasets decreased SABI's RMSE by 20.3\% and increased ECE by 2.9\% (see also Appendix~\ref{app:dmc}).

\begin{table}[!htbp]
\centering
\small
\setlength{\tabcolsep}{2pt}
\caption{Reference comparisons at $N=500$ (100 paired datasets) for \textbf{Experiment 3}. Expected calibration error (ECE) and RMSE for global means and variances; lower is better. SABI has substantially lower RMSE. Direct amortization also has low ECE, since its posteriors remain close to the prior.}
\label{tab:dmc-times}
\begin{tabular*}{\linewidth}{@{\extracolsep{\fill}}lllllll@{}}
\toprule
Method & Train & ESS/s & \multicolumn{2}{c}{ECE} & \multicolumn{2}{c}{RMSE}\\
\cmidrule(lr){4-5}\cmidrule(lr){6-7}
& (min) & & $\mu$ & $\sigma^2$ & $\mu$ & $\sigma^2$\\
\midrule
SABI (ours) & 179 & 32 & 0.042 & 0.059 & 0.120 & 0.307\\
Compositional & 292 & 12${}^{*}$ & 0.157 & 0.308 & 0.246 & 0.738\\
Direct amortized & 116 & 3853${}^{*}$ & 0.031 & 0.030 & 0.377 & 0.764\\
\bottomrule
\end{tabular*}
\par\smallskip
{\footnotesize\raggedright
${}^{*}$Nominal sample count from the network; SABI uses median bulk ESS.\par}
\end{table}

Finally, we fit the DMC to empirical data from the large ACDC dataset \citep{haaf2025acdc}. We use decision-making trials from the three most popular conflict tasks, resulting in 601,800 trials from 5,015 participants across 34 datasets and ten publications. For each DMC parameter, we specify a random-intercept model with crossed publication and task effects and nested dataset effects, yielding 312 global parameters and a total of 30,402 parameters (\autoref{fig:acdc-graph}). 

We only use SABI for the model of ACDC, as the competing methods already exhibited substantial recovery errors at much lower dimensionality. A key strength of SABI is that participant-level training is separate from the global hierarchy, allowing different joint models to reuse the network. Here we reuse the participant network in the richer model without retraining. 

Three chains each retained 2,500 draws in 10.79 minutes. The maximum $\widehat R$ was 1.037. Across global parameters, median bulk and tail ESS were 6,540 and 6,699, with minima of 71 and 23. Figure~\ref{fig:acdc-results}c shows automatic amplitude parameter $A$ across the 3 tasks, all 34 datasets and 5015 participants. See Appendix~\ref{app:dmc-pred} for posterior predictive checks.

\section{DISCUSSION}
In this work, we validated SABI, a semi-amortized Bayesian method for large graphical models. SABI combines the modeling flexibility of MCMC with fast, reusable neural conditional updates, without requiring a learned approximation of the entire joint posterior or post-hoc corrections. Across three experiments at a realistic scale, we obtained close agreement with Gibbs and NUTS reference hyperposteriors, faster inference than NUTS on a model with latent variables, and markedly better calibration than compositional inference under the same simulation budget.

\paragraph{Limitations}
SABI is not fully amortized per construction. Hence, each new dataset requires iterative sampling, warm-up, and convergence assessment. Computational gains must also offset simulation and training costs, as in all neural inference applications. As our learned updates receive no \textit{post hoc} correction, their approximation errors can bias the resulting posterior and are not eliminated by longer chains. Thus, accurate inference depends on conditional accuracy across the contexts encountered during sampling. Nonetheless, our results indicate clear improvement over existing amortized methods.

\paragraph{Future Directions}
One direction is to correct SABI's neural proposals using Metropolis-Hastings steps based on a reference posterior \citep{wang2018meta}, obtained from a separately trained estimator or from a teacher network. One could also learn local marginal likelihoods of the observed data given the conditioning parameters. MCMC would then use these learned factors to update the conditioning parameters without drawing the local parameters (Appendix~\ref{app:local-marginal-likelihoods}).

\section*{Acknowledgments}

N.S. and A.H. were supported by an Australia--US Multi-University Research Initiative grant (DSTG AUSMURIV000003, ONR/DoD N00014-23-1-2792). P.-C.B. acknowledges support from the Deutsche Forschungsgemeinschaft (DFG, German Research Foundation) through projects 508399956 and 528702768 and Collaborative Research Center 391 (Spatio-Temporal Statistics for the Transition of Energy and Transport), project 520388526. S.T.R. was funded by the National Science Foundation under Grant No.~2448380.

\par
\bibliographystyle{apalike}
\bibliography{references}

\clearpage
\setlength{\textfloatsep}{20pt plus 2pt minus 4pt}
\setlength{\floatsep}{12pt plus 2pt minus 2pt}
\setlength{\intextsep}{12pt plus 2pt minus 2pt}
\onecolumn
\appendix
\section{PARAMETER TRANSFORMATIONS}
\label{app:coordinates}
A learned block may be trained in coordinates $z_i$ related to the model parameters by a map $\theta_i=T_i(z_i)$. The map may depend on $(c_i,x_i)$ but is fixed during training, and the inputs needed to evaluate it are retained in the network input. For each $(c_i,x_i)$, assume the map is a differentiable bijection with differentiable inverse and nonsingular Jacobian. A density $q^z_{\phi_i}$ in the transformed coordinates then defines the estimator in model coordinates:
\begin{equation}
q_i(\theta_i\mid y_i,c_i,x_i)
=q^z_{\phi_i}\!\left(T_i^{-1}(\theta_i)\mid h_{\psi_i}(y_i,c_i,x_i)\right)
\left|\det D T_i^{-1}(\theta_i)\right|.
\label{eq:changeofvariables}
\end{equation}
Here $D$ denotes the Jacobian. Applying the same change of coordinates to the exact full conditional preserves both KL and total variation distance. With $z_i=T_i^{-1}(\theta_i)$, the expected negative log-likelihood in model coordinates differs from that in transformed coordinates by $\E\log|\det D T_i(z_i)|$. This expectation is over training simulations and is independent of $\phi_i$ and $\psi_i$ because the map is fixed during training.

For the Gaussian conditional prior~\eqref{eq:gaussian-conditional-prior}, the map $\theta_i=\mu_i+L_i z_i$, with $L_iL_i^\top=\Sigma_i$, gives prior-standardized coordinates and Jacobian factor $|\det L_i|^{-1}$. Draws are mapped back to $\theta_i$ before subsequent updates.

\section{REFINEMENT DIAGNOSTICS}
\label{app:refine}
For selected conditioning values and designs, we draw parameters from the conditional prior and simulate observations. Conditional calibration checks compare the generating parameters with independent neural posterior draws. Under exact inference, their ranks are uniform across simulation replicates \citep{talts2018sbc}. Checking separate conditioning regions and data summaries helps reveal errors that cancel when pooled.

Even simpler, sweep checks initialize simulated datasets at their generating parameter states and apply one or more inference sweeps. Exact posterior updates preserve the joint distribution of parameters and observations \citep{geweke2004getting}. Systematic changes in means, variances or dependencies across independent datasets therefore indicate error. Quantities conditioned on during simulation remain fixed during the check. These diagnostics can guide refinement by targeted simulation coverage in problematic regions.

\section{PROOFS AND SUPPLEMENTARY DERIVATIONS}
\label{app:proofs}

We use the parameter vector $u=(\theta,\vartheta)$, posterior distribution $\Pi$, and sweep kernels $P$ and $\widetilde P$ from Section~\ref{sec:theory}. Observations $y$ and model inputs $x$ remain fixed throughout. Index all blocks in the update order of Algorithm~\ref{alg:sampler}, so $u_i=\theta_i$ for $i\le N$ and $u_{N+j}=\vartheta_j$ for $j\le M$. Let $P_i$ be the transition kernel for block $i$ in the exact sampler and $\widetilde P_i$ the corresponding kernel in the approximate sampler. Conventional updates have $\widetilde P_i=P_i$. Then $P=P_1\cdots P_{N+M}$ and $\widetilde P=\widetilde P_1\cdots\widetilde P_{N+M}$. For probability distributions $\nu$ and $\nu'$, total variation distance is $\|\nu-\nu'\|_{\TV}=\sup_A|\nu(A)-\nu'(A)|$, where the supremum is over measurable sets $A$.

\subsection{Posterior invariance}
\label{app:joint-invariance}
An exact conditional draw replaces block $i$ while keeping the remaining blocks $u_{-i}$ fixed. Starting from $\Pi$, the resulting joint density is
\begin{equation}
p(u_{-i}\mid y,x)\,p(u_i\mid y,u_{-i},x)=p(u\mid y,x).
\end{equation}
A conventional MCMC transition also preserves this joint distribution when it preserves $p(u_i\mid y,u_{-i},x)$ for each fixed $u_{-i}$. Therefore $\Pi P_i=\Pi$ for every block update in the exact sampler, and composition gives $\Pi P=\Pi$ for any fixed update order. Under the standard MCMC conditions of Harris recurrence and aperiodicity for the joint chain with kernel $P$, its distribution converges to $\Pi$ in total variation from every initial state.

\subsection{Proof of Proposition~\ref{prop:finite-sweep}}
\label{app:finite-sweep}

\begin{proof}
Since $D_y<\infty$, each exact full conditional for a learned block is absolutely continuous with respect to its approximation for $\Pi$-almost every state. Start both chains from $\Pi$ and retain the states after every block update through $t$ sweeps. Since the initial distributions agree, the KL chain rule expresses the KL between these path distributions as the sum of the KL divergences between corresponding block transitions, averaged under the exact path distribution.

Conventional transitions are identical in both samplers and contribute zero. For learned block $i$, the exact and approximate transitions from state $u$ both leave $u_{-i}$ unchanged, so their KL equals $\kappa_i(u)$ from~\eqref{eq:conditional-error}. Every exact block update preserves $\Pi$, so every intermediate state of the exact chain has distribution $\Pi$. Each learned block is updated once per sweep, so the KL between the complete path distributions is
\begin{equation}
t\sum_{i=1}^N\E_\Pi[\kappa_i(u)]=tD_y.
\end{equation}
The distributions after $t$ sweeps are $\Pi$ and $\Pi\widetilde P^t$. Marginalizing the paths to their final states cannot increase KL, proving~\eqref{eq:finite-sweep-kl}. Pinsker's inequality and the upper bound of one on total variation distance give~\eqref{eq:finite-sweep-tv}.
\end{proof}

The bound accounts for errors in the learned update, including those caused by information discarded by the summary network. For a fixed summary network, the optimal unrestricted density estimator targets the posterior conditional on its output under the training distribution. This equals the full conditional only when the summary retains all information relevant to the block's posterior.

\subsection{Proof of Proposition~\ref{prop:stationary-error}}
\label{app:stationary-error}

\begin{proof}
We construct a family indexed by $0<\varepsilon<1/2$, suppressing this dependence in the notation. Let $\Pi$ be the posterior distribution of $\theta,\vartheta\in\{0,1\}$ with probability mass function
\begin{equation}
\begin{array}{c|cc}
p(\theta,\vartheta)&\vartheta=0&\vartheta=1\\ \hline
\theta=0&1/Z&\varepsilon/Z\\
\theta=1&\varepsilon^3/Z&\varepsilon/Z
\end{array}
\qquad Z=1+2\varepsilon+\varepsilon^3.
\end{equation}
A sweep first updates $\theta$ and then draws $\vartheta$ from its exact full conditional. Write
\begin{equation}
a=p(\theta=1\mid\vartheta=0)
  =\frac{\varepsilon^3}{1+\varepsilon^3},\qquad
b=p(\vartheta=1\mid\theta=0)
  =\frac{\varepsilon}{1+\varepsilon},\qquad
c=p(\vartheta=1\mid\theta=1)
  =\frac{1}{1+\varepsilon^2}.
\end{equation}
The remaining full conditional is $p(\theta=1\mid\vartheta=1)=1/2$. In the approximate sampler, replace the update of $\theta$ by a draw from
\begin{equation}
q(\theta=1\mid\vartheta=0)=a,
\qquad
q(\theta=1\mid\vartheta=1)=1-\varepsilon^2.
\end{equation}
Let $P$ and $\widetilde P$ denote the exact and approximate sweep kernels. Both have strictly positive transition probabilities on the four states, so both chains are irreducible and aperiodic and converge to unique stationary distributions. Write $\widetilde\Pi$ for the approximate sampler's stationary distribution.

Only the conditioning value $\vartheta=1$ contributes to $D_y$. Its posterior probability is $2\varepsilon/Z$, so, as $\varepsilon\to0$,
\begin{equation}
D_y
=-\frac{\varepsilon}{Z}
  \log\!\left[4\varepsilon^2(1-\varepsilon^2)\right]
\longrightarrow0.
\end{equation}

Under the approximate sweep, $\vartheta$ is a two-state Markov chain with transition probabilities from $0$ to $1$ and from $1$ to $0$ given by
\begin{equation}
\begin{aligned}
p_{01}&=(1-a)b+ac=\varepsilon+O(\varepsilon^2),\\
p_{10}&=\varepsilon^2(1-b)+(1-\varepsilon^2)(1-c)
       =2\varepsilon^2+O(\varepsilon^3).
\end{aligned}
\end{equation}
Its stationary probability of $\vartheta=1$ is therefore
\begin{equation}
\widetilde\Pi(\vartheta=1)
=\frac{p_{01}}{p_{01}+p_{10}}\longrightarrow1,
\qquad
\Pi(\vartheta=1)
=\frac{2\varepsilon}{Z}\longrightarrow0.
\end{equation}
The difference on this event bounds total variation distance from below. Since that distance is at most one, $\|\widetilde\Pi-\Pi\|_{\TV}\to1$.

The exact sweep does not depend on the initial $\theta$. Its two transition distributions, indexed by the initial $\vartheta$, differ in total variation by $|a-1/2|\le1/2$: the first draw has success probability $a$ or $1/2$, and the second uses the same full conditional given the updated $\theta$. Thus, for any initial distributions $\nu,\nu'$,
\begin{equation}
\|\nu P-\nu'P\|_{\TV}
\le\frac12\|\nu-\nu'\|_{\TV}.
\end{equation}
Taking $\varepsilon$ sufficiently small gives both inequalities in Proposition~\ref{prop:stationary-error} for any $\delta>0$. Thus the counterexample holds even when the exact sampler mixes rapidly.
\end{proof}

\subsection{Gaussian conditional structure}
\label{app:gaussian-conditional}

We study how errors in the prior mean supplied to a learned update affect its accuracy. 

Consider a learned block with conditional prior
\begin{equation}
\theta_i\mid c_i,x_i\sim\Normal(\mu_i,\Sigma_i),
\label{eq:gaussian-conditional-prior}
\end{equation}
where the current conditioning inputs determine the prior mean $\mu_i$ and covariance $\Sigma_i$. Errors in preceding updates can change these inputs and hence the prior supplied to the network. We compare the learned update supplied with prior mean $\mu_i+\delta\mu_i$ to the exact conditional posterior at $\mu_i$. Here $\delta\mu_i$ represents an error in the supplied prior mean, inherited from preceding updates.

Let $m_{q,i}(\mu_i)$ and $m_{p,i}(\mu_i)$ denote the learned and exact conditional posterior means. The resulting error, as calculated by the difference between the learned and exact means, separates into the network's own approximation error and the effect of the error in its supplied prior mean:
\begin{equation}
\begin{aligned}
&m_{q,i}(\mu_i+\delta\mu_i)-m_{p,i}(\mu_i)\\
&\quad=\big[m_{q,i}(\mu_i+\delta\mu_i)
             -m_{p,i}(\mu_i+\delta\mu_i)\big]\\
&\qquad+\big[m_{p,i}(\mu_i+\delta\mu_i)
             -m_{p,i}(\mu_i)\big].
\end{aligned}
\label{eq:gaussian-conditional-mean-error}
\end{equation}
The first bracket is the network's approximation error at the supplied prior mean. The second is the effect of the inherited error $\delta \mu_i$, the change in the exact posterior mean caused by the changed input.

The inherited error depends on how strongly the exact conditional posterior mean responds to its prior mean. For a Gaussian prior, this sensitivity is
\begin{equation}
\frac{\partial m_{p,i}}{\partial\mu_i^\top}
=\operatorname{Cov}_{p}(\theta_i)\Sigma_i^{-1},
\label{eq:gaussian-prior-response}
\end{equation}
where $\operatorname{Cov}_{p}(\theta_i)$ is the exact conditional posterior covariance \citep{giordano2018covariances}. For a small error in the supplied prior mean, this gives
\begin{equation}
\begin{aligned}
&m_{p,i}(\mu_i+\delta\mu_i)-m_{p,i}(\mu_i)\\
&\quad\approx\operatorname{Cov}_{p}(\theta_i)
       \Sigma_i^{-1}\delta\mu_i,
\end{aligned}
\label{eq:gaussian-inherited-mean-error}
\end{equation}
with the posterior covariance evaluated at the original $\mu_i$. For a scalar block, the inherited mean error is approximately $\delta\mu_i$ multiplied by the posterior-to-prior variance ratio. A ratio below one reduces the error; a ratio above one amplifies it.

\subsection{Stationary error and the Poisson equation}
\label{app:poisson}

The Poisson equation relates conditional approximation error to bias in posterior expectations \citep{rudolf2018perturbation}. For a function $f$, write $\Pi f=\E_\Pi[f(u)]$ and $Pf(u)$ for its expected value after one exact sweep starting from $u$. Suppose $\widetilde P$ has stationary distribution $\widetilde\Pi$ and there is a function $h_f$ satisfying the exact sampler's Poisson equation:
\begin{equation}
h_f-Ph_f=f-\Pi f.
\label{eq:poisson}
\end{equation}
Assume all expectations below are finite. Since $\widetilde\Pi h_f=\widetilde\Pi\widetilde P h_f$, stationarity gives
\begin{align}
\widetilde\Pi f-\Pi f
&=\widetilde\Pi(h_f-Ph_f)\nonumber\\
&=\widetilde\Pi(\widetilde P-P)h_f.
\label{eq:poisson-bias}
\end{align}
To identify the updates contributing to this bias, expand the difference between the sweep kernels as
\begin{equation}
\widetilde P-P
=\sum_{i=1}^{N+M}
\widetilde P_1\cdots\widetilde P_{i-1}
(\widetilde P_i-P_i)
P_{i+1}\cdots P_{N+M},
\label{eq:kernel-telescope}
\end{equation}
where empty products are identity kernels. Conventional transitions satisfy $\widetilde P_i=P_i$ and contribute zero. For learned block $i$, set $g=P_{i+1}\cdots P_{N+M}h_f$. The difference between its approximate and exact updates is
\begin{equation}
(\widetilde P_i-P_i)g(u)
=\int g(u_{-i},v)
\{q_i(v\mid y_i,c_i(u),x_i)-p(v\mid y_i,c_i(u),x_i)\}\,dv,
\end{equation}
where $g(u_{-i},v)$ denotes $g$ evaluated at $u$ with block $i$ replaced by $v$. Combining~\eqref{eq:poisson-bias} and~\eqref{eq:kernel-telescope} expresses stationary bias as a sum of signed contributions from conditional approximation errors. The preceding approximate updates determine the states at which each error is evaluated, while the subsequent exact updates determine its effect on $h_f$.

When the following series converges and $P$ can be applied term by term, the solution can be written as
\begin{equation}
h_f=\sum_{t=0}^\infty P^t(f-\Pi f).
\end{equation}
At state $u$, the term indexed by $t$ is the expected difference between $f$ and its posterior mean after $t$ exact sweeps. Slow decay allows conditional approximation errors to persist, while contributions with opposite signs can cancel.

\section{EXPERIMENTAL DETAILS}
\label{app:experiments}

We evaluated SABI across three applications to real data using high-dimensional models with complex dependencies. All experiments were run on a MacBook Pro with an Apple M4 Pro chip (14 CPU cores, 20 GPU cores) and 24 GB RAM.

\subsection{TIMSS}
\label{app:timss}
\label{app:TIMSS}

\subsubsection{Model and priors}
We use the first mathematics plausible value and home-resources score from \citep{timss2023home}. Excluding duplicate regional extracts and incomplete records leaves 256,416 students in 10,820 schools across 54 countries/systems. Both variables are standardized using their pooled, unweighted means and standard deviations; survey weights are not used.

For student $j$ in school $s$ and country $c$, the model is
\begin{equation}
\begin{aligned}
y_{csj}&\sim\Normal(\alpha_{cs}+\beta_{cs}x_{csj},\sigma^2),\\
\theta_{cs}=\begin{pmatrix}\alpha_{cs}\\\beta_{cs}\end{pmatrix}
&\sim\Normal_2(b_c,\Sigma_S),\qquad
b_c=\begin{pmatrix}\alpha_c\\\beta_c\end{pmatrix}
\sim\Normal_2(\mu,\Sigma_C).
\end{aligned}
\label{eq:timss-model}
\end{equation}
The covariance matrices $\Sigma_S$ and $\Sigma_C$ describe intercept--slope variation at the school and country levels. The independent global priors are
\begin{equation}
\mu\sim\Normal_2(0,I_2),\qquad
\Sigma_S,\Sigma_C\sim\operatorname{IW}_2\!\left(5,\operatorname{diag}(0.5,0.125)\right),\qquad
\sigma^2\sim\operatorname{IG}(3,1.28).
\end{equation}
The inverse-Wishart is parameterized by degrees of freedom and scale; the inverse-gamma by shape and scale.

\subsubsection{Training and inference}
For all three applications, exact training distributions and optimization settings are provided in the accompanying code.

Separate bivariate affine coupling flows estimate school and country coefficients, standardized by their conditional prior mean and Cholesky factor. The school input consists of $\log n$, $\bar x$, $\overline{x^2}$, $\bar y$, $\overline{xy}$, the country coefficients, $\Sigma_S$ and $\log\sigma$. The country input consists of the log number of schools, mean school coefficients, $\mu$, $\Sigma_S$ and $\Sigma_C$. Each covariance is encoded by two log standard deviations and a correlation, giving 11 inputs per network. Both flows have three affine coupling blocks with swap permutations; conditioner widths are $(128,128,128)$ for schools and $(64,64)$ for countries.

The country network trains for 600 epochs.

Both networks minimize posterior NLL. The school network trains for 8,700 epochs, followed by 200 epochs of refinement based on diagnostics of exploratory sweeps (Appendix~\ref{app:refine}).

Inference alternates school, country and conjugate global updates. Figure~\ref{fig:timss-global} and the two-sample C2ST use three chains of 30,000 sweeps, discarding 10,000 and retaining every 50th draw, giving 1,200 draws per method.

\subsection{Citi Bike}
\label{app:citibike}

\subsubsection{Model and priors}
The dataset covers 1 April--26 May 2024 and retains 1,427 stations with at least 50 departures in every week.

Let $y_{ith}$ denote departures at station $i$ in week $t$ and hour $h$. Each station--week has parameters $\theta_{it}=(a_{it},b_{it})^\top$, where $e^{a_{it}}$ is the mean hourly count averaged over the week and $b_{it}$ is log inverse overdispersion. The shared profile distributes the expected weekly total $H e^{a_{it}}$ across $H=168$ hours:
\begin{equation}
w_h=\frac{e^{u_h}}{\sum_{k=1}^{H}e^{u_k}},\qquad
\sum_{h=1}^{H}u_h=0.
\end{equation}
The hourly counts follow
\begin{equation}
\begin{aligned}
m_{ith}&=H e^{a_{it}}w_h,&
v_{it}&=\log(1+e^{-b_{it}}),\\
\ell_{ith}&\sim\Normal(\log m_{ith}-v_{it}/2,v_{it}),&
y_{ith}&\sim\operatorname{Poisson}(e^{\ell_{ith}}).
\end{aligned}
\label{eq:citibike-observation}
\end{equation}
Thus $\operatorname{E}(y_{ith}\mid\theta_{it},w)=m_{ith}$ and $\operatorname{Var}(y_{ith}\mid\theta_{it},w)=m_{ith}+m_{ith}^2e^{-b_{it}}$.

A Gaussian graph prior links a symmetrized four-nearest-neighbor spatial graph and consecutive weeks at each station. For $S$ stations and $T$ weeks, its precision across station--weeks is
\begin{equation}
Q=I_{ST}+\rho_S(I_T\otimes L_S)+\rho_T(L_T\otimes I_S),
\end{equation}
where $L_S$ and $L_T$ are the spatial and temporal graph Laplacians. Stacking station--week vectors within weeks gives mean $\mathbf1_{ST}\otimes\mu$ and covariance $Q^{-1}\otimes\Sigma$. Write $\Sigma=\operatorname{diag}(\tau_a,\tau_b)R\operatorname{diag}(\tau_a,\tau_b)$. We use independent priors
\begin{equation}
\begin{aligned}
\mu_a&\sim\Normal(0,2^2),&
\mu_b&\sim\Normal(\log5,1.5^2),\\
\log\tau_a,\log\tau_b&\sim\Normal(0,0.75^2),&
\log\rho_S,\log\rho_T&\sim\Normal(0,1.5^2),
\end{aligned}
\end{equation}
with $R\sim\operatorname{LKJ}(2)$. The hourly logits have a Gaussian prior on the sum-zero subspace,
\begin{equation}
u\sim\Normal_H\!\left(0,1.5^2[I_H-\mathbf1_H\mathbf1_H^\top/H]\right).
\end{equation}

For a station--week, let $n_{S,it}$ and $n_{T,it}$ count its spatial and temporal neighbors, and $s_{it}$ and $t_{it}$ sum their parameter vectors. Its conditional prior is
\begin{equation}
\begin{aligned}
d_{it}&=1+\rho_S n_{S,it}+\rho_T n_{T,it},\\
\eta_{it}&=\frac{\mu+\rho_S s_{it}+\rho_T t_{it}}{d_{it}},\\
\theta_{it}\mid\theta_{-(it)},\mu,\Sigma,\rho_S,\rho_T
&\sim\Normal_2(\eta_{it},\Sigma/d_{it}).
\end{aligned}
\label{eq:citibike-local-prior}
\end{equation}
The mean $\eta_{it}$ and covariance $\Sigma/d_{it}$ contain all graph information needed by the local network.

\subsubsection{Training and inference}
Both networks minimize conditional NLL on simulations. The \emph{station--week network} receives the 168 pairs $(\log(1+y_h),\log(Hw_h))$ and the conditional-prior means, log standard deviations and correlation. A SetTransformer with two 64-dimensional attention blocks, four heads and eight pooling seeds produces a 32-dimensional summary. Four spline coupling layers, with conditioner widths $(128,128)$, estimate the two prior-whitened parameters.

After 300 training epochs, we run a 100-station SABI pilot with four chains of 240 sweeps, discarding the first 80 and retaining every tenth state. The pilot supplies conditioning values and hourly profiles: 48 states from three chains form the training bank, and 16 states from the fourth chain form the validation bank. For a further 100 epochs, training mixes the initial simulation design equally with perturbed pilot conditioning values and profiles. Station--week parameters are freshly drawn from the corresponding conditional priors, and counts are freshly simulated. Both banks use the same observed dataset.

The \emph{hourly-profile network} estimates $\log(w_h/w_j)$ at fixed pair mass $w_h+w_j$. Gaussian augmentation of the profile prior permits parallel updates of disjoint pairs. Inputs combine the two hours' transformed count residuals and station--week parameters with the augmented prior context and moment summaries. A two-block DeepSet of width 64, with four-head, four-seed attention pooling, produces a 64-dimensional summary. Four scalar spline layers use 16 bins and conditioner widths $(128,128)$.

The station--week network is reused across station counts. A separate hourly-profile network trains for 300 epochs for each station count. Inference alternates both neural updates with conventional updates of the graph parameters.

\subsubsection{Posterior predictives}
\label{app:citi-pred}
The predictive check uses the full 1,427-station fit with four chains of 1,250 retained draws. From each chain we select 256 states without replacement, generate lognormal rates and Poisson counts, and sum every 24 hours into daily totals. The median and pointwise 95\% intervals summarize these 1,024 replications. All eight weeks enter the fit.

Figure~\ref{fig:citibike-predictive} shows stations near the 3rd, 50th, 75th and 90th percentiles of observed activity. Their daily interval coverages are 94.6\%, 87.5\%, 80.4\% and 83.9\%, respectively: the model understates daily variation at the busier selected stations.
\begin{figure}[!htbp]
\centering
\includegraphics[width=\linewidth]{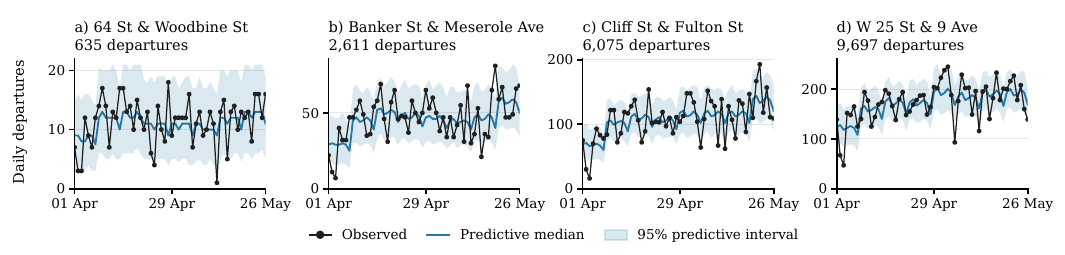}
\caption{Citi Bike daily departures at four stations, ordered by observed activity. Predictive medians and pointwise 95\% intervals include parameter and observation variation. All 56 days are used for fitting.}
\label{fig:citibike-predictive}
\end{figure}

\subsection{DMC and ACDC}
\label{app:dmc}

\subsubsection{Model and priors}
The DMC combines controlled evidence accumulation with a transient automatic response \citep{ulrich2015dmc}. Following \citet{schaefer2026dmc}, we simulate
\begin{equation}
\begin{aligned}
\mathrm dX(t)&=[\mu_c+c\,h'(t)]\mathrm dt+\mathrm dW(t),\\
h(t)&=A(t/\tau)\exp(1-t/\tau),\qquad X(0)=0,
\end{aligned}
\label{eq:dmc-process}
\end{equation}
where $c=+1$ for congruent and $c=-1$ for incongruent trials. The first crossing of $+b$ or $-b$ determines a correct or incorrect response. Response time adds a positive-truncated $\Normal(\mu_r,\sigma_r^2)$ nondecision draw to the crossing time. All six participant parameters are positive (Table~\ref{tab:dmc-priors}). The starting point is fixed at the midpoint, pulse shape at two and diffusion noise at one. Euler simulation uses 1\,ms steps, right-endpoint drift and boundary detection on the same grid.

Let $\theta_i$ collect the six log parameters. In the calibration model, $\theta_{ik}\sim\Normal(\mu_k,\sigma_k^2)$ independently across coordinates. Table~\ref{tab:dmc-priors} gives the population priors. The ACDC model adds publication, dataset and task effects:
\begin{equation}
\begin{aligned}
\theta_{ik}&\sim\Normal(\mu_k+u_{pk}+v_{dpt,k}+w_{tk},\sigma_k^2),\\
u_{pk}&\sim\Normal(0,\omega_{P,k}^2),\qquad
v_{dpt,k}\sim\Normal(0,\omega_{D,k}^2),\qquad
w_{tk}\sim\Normal(0,\omega_{T,k}^2).
\end{aligned}
\label{eq:acdc-hierarchy}
\end{equation}
Each dataset belongs to one publication and task. The priors for $\mu_k$ and $\sigma_k$ remain those in Table~\ref{tab:dmc-priors}, with independent $\log\omega_{P,k},\log\omega_{D,k},\log\omega_{T,k}\sim\Normal(\log0.15,0.5^2)$.

\begin{table}[ht]
\centering
\caption{DMC population priors. Normal distributions are parameterized by mean and variance; physical time parameters are measured in seconds.}
\label{tab:dmc-priors}
\begin{tabular}{llll}
\toprule
Parameter & Meaning & Prior for $\mu_k$ & Prior for $\log\sigma_k$\\
\midrule
$A$ & Automatic amplitude & $\Normal(\log 0.20,0.35^2)$ & $\Normal(\log 0.40,0.25^2)$\\
$\tau$ & Automatic peak time & $\Normal(\log 0.15,0.30^2)$ & $\Normal(\log 0.35,0.25^2)$\\
$\mu_c$ & Controlled drift & $\Normal(\log 3.00,0.35^2)$ & $\Normal(\log 0.35,0.25^2)$\\
$\mu_r$ & Nondecision location & $\Normal(\log 0.35,0.20^2)$ & $\Normal(\log 0.22,0.25^2)$\\
$b$ & Boundary distance & $\Normal(\log 0.60,0.35^2)$ & $\Normal(\log 0.35,0.25^2)$\\
$\sigma_r$ & Nondecision spread & $\Normal(\log 0.04,0.30^2)$ & $\Normal(\log 0.35,0.25^2)$\\
\bottomrule
\end{tabular}
\end{table}

The empirical analysis retains participants with accuracy above 0.75 and median RT below 1.25\,s, restricts RTs to $[0.2,3)$\,s, and selects 60 trials per congruency condition. A narrow spike of incorrect Rey-Mermet trials in $[1.995,1.996)$\,s is removed before selection. The resulting data contain 5,015 dataset-specific participant units and 601,800 trials across 34 datasets, ten publications and three tasks.

\subsubsection{Training and inference}
The \emph{participant network} receives six conditional-prior means and six log standard deviations. It estimates the six prior-standardized log parameters using eight spline coupling layers with conditioner widths $(128,128)$. A 64-dimensional trial summary combines a SetTransformer with response- and congruency-specific RT-distribution features. Each simulation draws participant parameters and 120 trials, equally split by congruency.

\begin{table}[htbp]
\centering
\small
\caption{DMC recovery: 95\% coverage; RMSE in prior-SD units.}
\begin{tabular}{lrrrrrrrrr}
\toprule
& \multicolumn{3}{c}{SABI}
& \multicolumn{3}{c}{Compositional}
& \multicolumn{3}{c}{Direct amortized} \\
\cmidrule(lr){2-4}\cmidrule(lr){5-7}\cmidrule(lr){8-10}
& Coverage & RMSE & ECE
& Coverage & RMSE & ECE
& Coverage & RMSE & ECE \\
\midrule
$\mu_A$               & 91 & 0.209 & 0.065 & 83 & 0.194 & 0.117 & 97 & 0.216 & 0.057 \\
$\mu_\tau$            & 89 & 0.256 & 0.051 & 59 & 0.489 & 0.279 & 96 & 0.519 & 0.015 \\
$\mu_{\mu_c}$         & 95 & 0.052 & 0.025 & 78 & 0.123 & 0.131 & 97 & 0.236 & 0.041 \\
$\mu_{\mu_r}$         & 93 & 0.061 & 0.061 & 77 & 0.138 & 0.086 & 99 & 0.232 & 0.039 \\
$\mu_b$               & 94 & 0.062 & 0.022 & 87 & 0.095 & 0.095 & 97 & 0.148 & 0.017 \\
$\mu_{\sigma_r}$      & 95 & 0.081 & 0.026 & 66 & 0.437 & 0.233 & 94 & 0.911 & 0.019 \\
\midrule
$\sigma_A^2$          & 90 & 0.321 & 0.041 & 40 & 0.908 & 0.399 & 96 & 0.958 & 0.036 \\
$\sigma_\tau^2$       & 77 & 0.815 & 0.203 & 40 & 1.016 & 0.345 & 96 & 0.996 & 0.020 \\
$\sigma_{\mu_c}^2$    & 90 & 0.142 & 0.025 & 55 & 0.644 & 0.285 & 97 & 0.699 & 0.022 \\
$\sigma_{\mu_r}^2$    & 93 & 0.148 & 0.014 & 55 & 0.554 & 0.225 & 94 & 0.596 & 0.019 \\
$\sigma_b^2$          & 91 & 0.185 & 0.030 & 63 & 0.550 & 0.273 & 95 & 0.502 & 0.044 \\
$\sigma_{\sigma_r}^2$ & 93 & 0.231 & 0.043 & 49 & 0.754 & 0.320 & 95 & 0.833 & 0.036 \\
\bottomrule
\label{tab:dmc-full}
\end{tabular}
\end{table}

The participant network minimizes posterior NLL for 600 epochs, followed by 200 epochs of refinement based on sweep checks and simulation-based calibration (Appendix~\ref{app:refine}).

The \emph{compositional estimator} pools five participants' 600 trials per factor and trains for 640 epochs. A SetTransformer with two 64-dimensional attention blocks, four heads and a 16-dimensional output conditions an EDM diffusion network with five 256-unit layers, using F prediction and noise loss. The \emph{direct estimator} pools all 60,000 trials of each 500-participant dataset without participant identifiers and trains for 100 epochs. A SetTransformer with 32 inducing points and a 16-dimensional output conditions an eight-layer spline flow with widths $(128,128)$.

Both baselines use the priors in Table~\ref{tab:dmc-priors} and estimate the 12 analytically prior-standardized population parameters. Training budgets are closely matched: SABI and compositional inference each use 26,214,400 participant simulations, and direct inference uses 26,200,000.

Evaluation uses 100 paired datasets of 500 participants and 120 trials each. SABI runs three chains of 2,000 sweeps, discards 500 and retains every fifth draw, giving 900 draws per dataset. Both baselines retain 801 draws, matching SABI's median bulk ESS. Compositional sampling uses the adaptive two-step diffusion integrator with score mini-batches and denoised-prediction clipping. The empirical ACDC fit reuses the same participant network with three 6,000-sweep chains, discarding 1,000 and retaining every second draw.

Table~\ref{tab:dmc-full} shows 95\% coverage, RMSE and calibration error for all global parameters. RMSE of posterior means is divided by the analytic prior standard deviation of the corresponding parameter. Calibration error is the median absolute difference between empirical and nominal coverage over 20 equally spaced equal-tailed interval levels from 0.005 to 0.995. Both metrics are then averaged over the six coordinates in each family.

\subsubsection{Posterior predictives}
\label{app:dmc-pred}
We select 100 evenly spaced retained states from each ACDC chain. At each state, conditional draws of the fitted participants' parameters generate 120 trials per participant. Figure~\ref{fig:acdc-predictive} summarizes these 300 replications by task. Conditional accuracy uses five equal-count RT bins within each participant and congruency condition. Delta plots subtract congruent from incongruent correct-RT quantiles at probabilities $0.1,0.3,0.5,0.7,0.9$. Both summaries weight participants equally; shaded regions are pointwise 95\% predictive intervals.

Predictions reproduce the increasing Stroop delta curve, but underestimate the slowest-quantile RT difference in all three tasks. Predicted incongruent Stroop accuracy rises with RT, whereas observed accuracy is nearly flat or decreasing.
\begin{figure}[!htbp]
\centering
\includegraphics[width=0.9\linewidth]{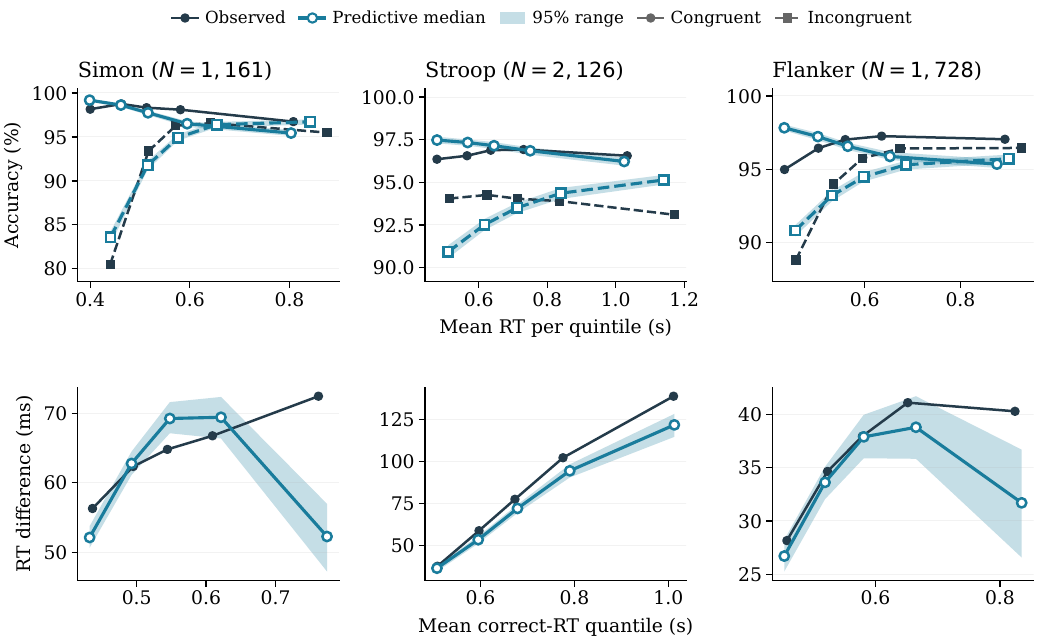}
\caption{ACDC posterior predictives by task. Top: accuracy across five within-participant RT bins, separately for congruent and incongruent trials. Bottom: incongruent-minus-congruent correct-RT quantiles against their average RT. Observed summaries are compared with predictive medians and pointwise 95\% intervals from 300 replications.}
\label{fig:acdc-predictive}
\end{figure}

\section{LEARNING LOCAL MARGINAL LIKELIHOODS}
\label{app:local-marginal-likelihoods}

For local blocks that are conditionally independent given shared parameters $\vartheta$, an alternative to SABI's local posterior updates is to learn the marginal likelihood factors
\begin{equation}
m_i(y_i\mid c_i,x_i)
=\int p(y_i\mid\theta_i,c_i,x_i)\,
p(\theta_i\mid c_i,x_i)\,d\theta_i,
\end{equation}
where $c_i=c_i(\vartheta)$ contains the conditioning parameters for block $i$.

Training uses the same local simulations as SABI: draw $\theta_i$ from its conditional prior and simulate $y_i$. Instead of learning the conditional posterior of $\theta_i$, train a conditional density estimator $\widehat m_i$ for $y_i$ given $c_i,x_i$. MCMC then targets
\begin{equation}
\widehat p(\vartheta\mid y,x)
\propto p(\vartheta\mid x)
\prod_i \widehat m_i(y_i\mid c_i(\vartheta),x_i),
\end{equation}
evaluating the learned factors at each proposed state without drawing the local parameters. This reduces the sampled parameter space, but requires density estimation over the local data and introduces approximation error through the learned likelihood factors.

\end{document}